\documentclass{article}
\usepackage{arxiv}

\newcommand{\ie}{i.\,e.\xspace}

\newcommand{\eg}{e.\,g.\xspace}

\newcommand{\qdga}{QD-GA\xspace}
\newcommand{\qdea}{QD-EA\xspace}

\newcommand{\QDGA}{QD-GA\xspace}

\usepackage{xcolor}
\usepackage{natbib}

\usepackage{newfloat}
\usepackage{listings}
\usepackage{amsmath,amssymb,amsthm}
\usepackage[vlined,linesnumbered,ruled,titlenotnumbered]{algorithm2e}

\allowdisplaybreaks

\usepackage{amssymb,amsfonts,dsfont,amsmath,amsthm}

\usepackage{tikz}
\usetikzlibrary{shapes,arrows}
\usetikzlibrary{decorations}
\usetikzlibrary{plotmarks}
\usetikzlibrary{calc}
\usetikzlibrary{mindmap}
\usetikzlibrary{shadows}
\usetikzlibrary{backgrounds}
\usetikzlibrary{patterns}
\usetikzlibrary{shapes.symbols}

\newcommand{\ignore}[1]{}
\usepackage{graphicx}

\newtheorem{lemma}{Lemma}
\newtheorem{corollary}{Corollary}
\newtheorem{theorem}{Theorem}
\usepackage[vlined,linesnumbered,ruled,titlenotnumbered]
{algorithm2e}

\title{Theoretical Analysis of Quality Diversity Algorithms for a Classical Path Planning Problem}

\author{
Duc-Cuong Dang\\
Algorithms for Intelligent Systems\\
University of Passau\\
Passau, Germany
\And
Aneta Neumann\\
Optimisation and Logistics\\
School of Computer and Mathematical Sciences\\
The University of Adelaide\\
Adelaide, Australia
\And
Frank Neumann\\
Optimisation and Logistics\\
School of Computer and Mathematical Sciences\\
The University of Adelaide\\
Adelaide, Australia
\And
Andre Opris\\
Algorithms for Intelligent Systems\\
University of Passau\\
Passau, Germany
\And
Dirk Sudholt\\
Algorithms for Intelligent Systems\\
University of Passau\\
Passau, Germany
}

\begin{document}

\maketitle
\begin{abstract}
  Quality diversity (QD) algorithms have shown to provide sets of high quality solutions for challenging problems in robotics, games, and combinatorial optimisation. So far, theoretical foundational explaining their good behaviour in practice lack far behind their practical success. We contribute to the theoretical understanding of these algorithms and study the behaviour of QD algorithms for a classical planning problem seeking several solutions. 
We study the all-pairs-shortest-paths (APSP) problem which gives a natural formulation of the behavioural space based on all pairs of nodes of the given input graph that can be used by Map-Elites QD algorithms.
Our results show that Map-Elites QD algorithms are able to compute a shortest path for each pair of nodes efficiently in parallel. Furthermore, we examine parent selection techniques for crossover that exhibit significant speed ups compared to the standard QD approach.
\end{abstract}

\section{Introduction}

In recent years, computing diverse sets of high quality solutions for combinatorial optimisation problems has gained significant attention in the area of artificial intelligence  
from both
theoretical~\citep{DBLP:journals/ai/BasteFJMOPR22,DBLP:journals/algorithms/BasteJMPR19,DBLP:journals/algorithmica/FominGJPS24,DBLP:conf/aaai/HanakaK0KKO23} 
and experimental~\citep{DBLP:conf/icrai/VonasekS18,DBLP:conf/aaai/IngmarBST20} 
perspectives. 
Prominent examples where diverse sets of high quality solutions are sought come from the area of path 
planning~\citep{DBLP:conf/aaai/HanakaKKO21,DBLP:conf/latin/00010STTY22}.
Particularly, quality diversity (QD) algorithms have shown to 
produce excellent results for challenging problems in the areas such as robotics~\citep{DBLP:conf/icra/MiaoZSZZHYW22,DBLP:conf/ijcai/ShenZHMCFL20}, games~\citep{DBLP:journals/tec/CullyD18} and combinatorial optimisation~\citep{DBLP:journals/telo/NikfarjamNN24}.

This work contributes to the theoretical understanding of QD algorithms. 
Such algorithms compute several solutions that occupy different areas of a so-called behavioural space. 
Approaches that use a multidimensional archive of phenotypic elites, called Map-Elites~\citep{DBLP:journals/corr/MouretC15}, are among the most commonly used QD algorithms. 
We present the first runtime analysis of Map-Elites QD algorithms for a classical 
path planning 
problem, namely the all-pairs shortest-path (APSP) problem, that seeks to find a set of 
shortest paths for all possible source-destination pairs in a graph. 
From a QD perspective, it is natural to interpret the behavioural space for this problem as the set of all source-destination pairs. 
This makes the problem an ideal example for a theoretical study of Map-Elites approaches. Our aim is being able to provide insights into their working principles.

\subsection{Related work} 
%
%
Evolutionary computation methods provide a flexible way of generating diverse sets of high-quality solutions by directly incorporating diversity measures into the population. Generating diverse and high-quality solution sets has gained significant interest in evolutionary computation, particularly under the notion of evolutionary diversity optimisation (EDO)~\citep{DBLP:conf/gecco/GaoN14,DBLP:journals/ec/GaoNN21,DBLP:conf/gecco/NeumannGDN018} and quality diversity (QD)~\citep{DBLP:conf/gecco/LehmanS11,DBLP:journals/corr/MouretC15,DBLP:conf/ppsn/HaggAB18}.

In the context of combinatorial optimisation, EDO methods have been developed to compute sets of problem instances~\citep{DBLP:conf/gecco/NeumannGDN018,DBLP:conf/gecco/NeumannG0019,DBLP:journals/ec/GaoNN21} that are capable of distinguishing 
two algorithms by their performances. 
More recently, this approach has been applied to compute diverse sets of high-quality solutions for combinatorial optimisation problems such as 
  the knapsack problem~\citep{DBLP:conf/gecco/BossekN021}, 
  the traveling salesperson problem~\citep{DBLP:conf/gecco/DoBN020,DBLP:conf/gecco/NikfarjamBN021}, 
  the computation of minimum spanning trees~\citep{DBLP:conf/gecco/Bossek021}, 
  the maximum matching problem~\citep{PPSN2024_EDO_MMP}, 
  communication problems in the area of defence~\citep{DBLP:conf/gecco/NeumannGYSCG023}, 
  and constrained monotone submodular functions~\citep{DBLP:conf/gecco/NeumannB021}.

EDO algorithms maintain a fixed population size and focus on maximising diversity based on a specified diversity metric by ensuring that all solutions meet a given quality criteria. In contrast, QD algorithms typically utilise a variable population size to identify optimal solutions across different niches within a specified behavioural space. Both approaches have initially been applied to design problems~\citep{DBLP:conf/ppsn/HaggAB18,DBLP:conf/gecco/NeumannG0019}. 
In fields such as robotics and gaming, several variants of QD algorithms have been developed~\citep{DBLP:conf/ppsn/PughSS16,DBLP:conf/gecco/GravinaLY18,DBLP:conf/gecco/FontaineTNH20,DBLP:conf/gecco/ZardiniZZIF21,DBLP:conf/aaai/FontaineLKMTHN21,DBLP:conf/gecco/MedinaRMS23}. 
Moreover, QD has been used to evolve instances for the traveling salesperson problem~\citep{DBLP:conf/gecco/Bossek022}, for solving the traveling thief problem~\citep{DBLP:journals/telo/NikfarjamNN24}, and in context of time-use optimisation to improve health outcomes~\citep{DBLP:conf/gecco/NikfarjamSND024}.


Rigorous theoretical studies of QD algorithms have 
only been started recently. 
The first runtime analysis has been provided for the classical knapsack problem~\citep{DBLP:conf/ppsn/NikfarjamDN22} where it has been shown that QD is able to solve this problem in expected pseudo-polynomial time. 
Afterwards, runtime results have been 
proven for the computation of minimum spanning trees and the optimisation of submodular monotone functions under cardinality constraints~\citep{Bossek2024} and its generalisation to approximately submodular functions as well as the classical set cover problems~\citep{DBLP:journals/corr/abs-2401-10539}.
The mentioned studies have in common that they only seek on a single solution. To our knowledge, the only study analysing QD algorithms for problems with multiple solutions has been carried out by \citet{DBLP:conf/gecco/SchmidbauerOB0S24}. The authors considered monotone submodular functions with different artificial Boolean constraints which defined subproblems that served as stepping stones to solve the main optimisation problem.

\subsection{Our contribution}
In general, QD approaches aim to solve problems where several solutions that belong to different areas of a behaviour space are sought. 
We provide a first runtime analysis for a natural and classical problem seeking several solutions by studying QD approaches for the classical APSP problem. Note that the APSP can be solved in polynomial time using classical algorithms~\citep{10.1145/367766.368168,10.1145/321992.321993}. The goal of the investigations of evolutionary computation techniques for this problem is not to beat these classical algorithms in terms of runtime, but to provide a theoretical understanding of their working behaviour for this important fundamental problem.

Our work is built upon the previous analysis from~\citep{DBLP:journals/tcs/DoerrHK12} for the so-called ($\leq \mu$+$1$)~GA/EA algorithms where they showed that the expected runtime can be sped up to $O(n^{3.5}\sqrt{\log{n}})$ from $\Omega(n^4)$ when enabling crossover (or using the GA variant). The upper bound was 
improved to $O(n^{3.25} \sqrt[4]{\log{n}})$ in~\citep{DoerrT09} with a more refined analysis (which was also shown to be asymptotically tight in the worst case),
and later to $O(n^3\log{n})$ in~\citep{DBLP:journals/tcs/DoerrJKNT13} using better operators. 
Despite some criticisms of the ($\leq \mu$+$1$) scheme at the time, \eg see \citep{Corus2018}, we believe such scheme is natural in the modern view of QD algorithms because not-yet existing solutions 
(those not-yet occupying their slots) 
in the population of the GA/EA can be seen as equivalent to empty cells in the archive of a Map-Elites algorithm.
We therefore define and analyse the so-called QD-GA algorithms, in which a map of $n$ by $n$ is maintained to store the best-so-far shortest paths between all source-destination pairs for APSP. In each iteration, crossover or mutation operators as defined in~\citep{DBLP:journals/tcs/DoerrHK12} is exclusively applied from parents uniformly selected from the cells to create a new offspring solution which can be used to update the map. When only mutation is used, the algorithm
is referred to as QD-EA. 

Using a similar approach to in the previous work, however with a slightly different tool from~\citep{DBLP:journals/ipl/Witt14}, we show that QD-EA optimises APSP in $O(n^2\Delta \max\{\ell,\log{n}\})$ in fitness evaluations in expectation, 
here $\Delta$ is the maximum degree and $\ell$ is the diameter of the graph respectively. On the other hand, QD-GA with a constant probability of applying crossover can achieve expected runtime $O(n^{3.5}\sqrt{\log{n}})$. We also give a proof sketch for the improved tight bound of $O(n^{3.25}\sqrt[4]{\log{n}})$ for \qdga. 
We further prove that if the selection of parent cells is restricted to only include \emph{compatible} parents for crossover from the map, QD-GA only relying on crossover optimises APSP in $O(n^{3}\log{n})$ expected fitness evaluations, and therefore when combining with mutation, the expected runtime is $O(\min\{n^2\Delta \max\{\ell,\log{n}\},n^{3}\log{n}\})$. 
We refer to this latter version of QD-GA as the Fast QD-APSP algorithm. 
Our proven results not only hold in expectation but also with high probability, namely probability $1-o(1)$. 
Instances in which the lower bounds on the runtime hold tight with respect the to proven upper bounds are also discussed. 
In a related work, \citep{ant-sp-journal} proved the expected runtime $O(n^3\log^3{n})$ for ant colony optimisation algorithms with information sharing on APSP, however this is a different type of algorithms
compared to 
EAs or QDs.

The paper is structured as follows. In the next section,
we present the quality diversity approach for the APSP problem that is subject to our investigations.
We show that this approach solves this classical problem in expected polynomial time and discuss speed-ups through more tailored parent selection methods for crossover afterwards.
Finally, we finish with some concluding remarks.

\section{Quality Diversity and the APSP}
The all-pairs-shortest-path (APSP) problem is 
a classical combinatorial optimisation problem. 
Given a directed strongly connected graph $G(V, E)$ with $n=|V|$ and a weight function $w \colon E \rightarrow \mathds{N}$ 
on the edges, the goal 
is to compute for any given pair of nodes $(s, t) \in V \times (V \setminus \{s\})$, a shortest path (in terms of the weight of the chosen edges) from $s$ to $t$. 
If one only looks for the shortest path between two specific nodes, the problem is referred as the single-source single-destination shortest-path problem (SSSDSP). 
Like in \citep{DBLP:journals/tcs/DoerrHK12}, we assume that 
  $G$ is strongly connected, thus there exists a path from $s$ to $t$ for any distinct pair $(s,t)$ of nodes.

        \begin{algorithm}[t]     
         Initialise an empty map $M$ \\
         Set $M_{s,t}=(s,t), \forall (s,t) \in E$\;
\While{termination condition not met}{
                Choose $p \in [0,1]$ uniformly at random\;
                \If{$p \leq p_c$}{
                    Choose $I_{st}, I_{uv} \in M$ uniformly at random\;
                    Generate $I'$ from $I_{st}$ and $I_{uv}$ by crossover\;
                }\Else{
                    Choose an individual $I_{st} \in M$ uniformly at random\;
            
                Generate $I'$ from $I_{st}$ by mutation\;
                }
                $M \gets \textnormal{Update}(M, I')$\;
            } 
            \caption{Quality Diversity Genetic Algorithm (\QDGA) 
            with Crossover Probability $p_c$}
            \label{alg:QDGA}
        \end{algorithm}

        \begin{algorithm}[t]
        
                \If{$I'$ is a valid path from $s$ to $t$}{
                    \lIf{$M_{st}$ is empty}{
                        store $I'$ in cell $M_{st}$\!
                    }\Else{
                        Let $I$ be the search point in cell $M_{st}$\;
                        \lIf{$f_{st}(I') \leq f_{st}(I)$}{
                           replace $I$ by $I'$ in $M_{st}$ \!
                        }
                    }
                }

             \caption{Update Procedure for Minimization}
            \label{alg:Update}
        \end{algorithm}

In light of QD approach, we see solving APSP as evolving both diverse and high quality solutions of multiple SSSDSP. 
Therefore, a valid search point $I$ is 
a set of chosen edges, 
denoted by $E(I)$, that forms 
a valid path from a starting node $s$ to a target node $t$. 
This path can also be represented by a sequence of nodes to visit, \ie 
$I:=(s=v_1,v_2,\ldots,v_k=t)$ such that $E(I):=\{(v_{i},v_{i+1})\mid i\in [k-1]\} \subseteq E$ where $[n]$ denotes the set $\{1,\ldots,n\}$ for $n \in \mathbb{N}$. 
By $|I|$ we denote its \emph{cardinality}, 
that is, $|I|:=|E(I)|=k-1$. 
Occasionally, to make the source $s$ and destination $t$ clear, we also write such path $I$ as $I_{st}$. For storage, either the sequence $I$ or the set $E(I)$ can be stored in memory as either one of these uniquely defines $I$. For a given search point $I$, let 
\[
f_{st}(I) = \sum_{e \in E(I)} w(e)
\]
be the weight (or length) of $I$. 
A shortest path between $s$ and $t$ is the one that minimises $f_{st}$, and if this path is the above path $I$ then $I_{v_i v_{i+j}}=(v_i,v_{i+1},\ldots,v_{i+j})$ for any $i\in [k-1]$ and any $j<k-i$ is also a shortest path between $v_i$ and $v_{i+j}$.

We use the Quality Diversity Genetic Algorithm (QD-GA) given in Algorithm~\ref{alg:QDGA} for our investigations. The algorithm works with a 2D map $M$ of dimension $n$ by $n$ excluding the diagonal. 
The rows and columns of $M$ are indexed by nodes of $V$, and each cell $(s,t) \in V \times (V \setminus \{s\})$ in the map  can store an individual $I_{st}$ representing a path that starts at $s$ and ends in $t$. 

Following the same approach 
as \citep{DBLP:journals/tcs/DoerrHK12} for the initialisation of their population, 
our map $M$ is initialised 
to contain for each cell $M_{st}$ where $(s, t) \in E$.
the solution $I$ with $E(I)=\{(s,t)\}$ consisting of the path from $s$ to $t$ given by the edge $(s,t)$, while the other cells of $M$ are initialised to empty.
In each iteration, either mutation or crossover is used to produce a new individual from the individuals of the current map $M$.
Crossover is carried out in each iteration with probability $p_c$ and chooses two individuals from $M$ uniformly at random to produce an offspring $I'$.
If crossover is not applied then mutation is carried out on a uniformly at random chosen individual to produce $I'$. $I'$ is then used to update the map $M$ using Algorithm~\ref{alg:Update} in the case that $I'$ is a valid path. 
If $I'$ is a valid path from some node $s$ to some node $t$, then $I'$ is introduced into the cell $M_{st}$ if $M_{st}$ is currently empty.
Otherwise if $M_{st}$ is not empty then $I'$ replaces $I_{st}$ at position $M_{st}$ iff $f(I') \leq f(I_{st})$ holds. 

As common in the theoretical analysis of evolutionary algorithms, we measure the runtime of Algorithm~\ref{alg:QDGA} by the number of fitness evaluations. The \emph{optimisation time} of Algorithm~\ref{alg:QDGA} refers to the number of fitness evaluations, until $M$ contains for each pair $(s, t) \in V \times (V \setminus \{s\})$ a shortest path from $s$ to $t$. The \emph{expected optimisation time} refers to the expectation of this value. 

This optimisation time can depend on the following characteristics of the graph: 
  the maximum degree $\Delta:=\max_{u\in V}|\{(v\mid (u,v)\in E \vee (v,u) \in E\}|$,
  and the largest cardinality of the shortest paths $\ell:=\max_{(s,t)\in V\times (V\setminus\{s\}),I\in\mathcal{I}(s,t)}\{|I| \mid f_{st}(I) \text{ is minimised}\}$, where $\mathcal{I}(s,t)$ denotes set of all valid paths from $s$ to $t$.
For unweighted graphs, this parameter $\ell$ is known as the \emph{diameter}.

\subsection{Mutation and Crossover}
We use the following mutation and crossover operators introduced in~\citep{DBLP:journals/tcs/DoerrHK12}. 
For a given path $I_{st}$, starting at node $s$ and ending at node $t$, let $E_{st}$ be the set of all edges incident to $s$ or $t$. 
Mutation relies on the elementary operation of choosing an edge $e \in E_{st}$ uniformly at random dependent on a given individual $I'_{st}$. If $e$ is part of $I'$, then $e$ is removed from $I'$. Otherwise, $e$ is added $I'$ with potentially extends the path at its start or end node.  
The offspring $I'$ is obtained by creating a copy of $I_{st}$ and applying this elementary operation $k+1$ times to $I'$ where $k$ is chosen according to a Poisson distribution 
$Pois(\lambda)$ 
with parameter $\lambda=1$. Note that an elementary operation might create an invalid individual when adding an outgoing edge to $s$ or an incoming edge to $t$. Such invalid individuals are rejected by the algorithm according to the update procedure given in Algorithm~\ref{alg:Update}.

Our analysis will rely on mutation steps carrying out a single valid elementary operation extending a given path. For crossover, we choose two $I_{st}$ and $I_{uv}$ from the current map $M$ uniformly at random. If $t=u$, then the resulting offspring $I$ is obtained by appending $I_{uv}$ to $I_{st}$ and constitutes a path from $s$ to $v$.
If $I$ does not constitute a path, then the individual is discarded and has no effect during the update procedure of Algorithm~\ref{alg:Update}.

\subsection{Analytical tools} 

We use the following tail bound from \citep{DBLP:journals/ipl/Witt14} for sum of geometric variables in our analyses. 

\begin{lemma}[Theorem~1 in \citep{DBLP:journals/ipl/Witt14}]\label{lem:sum-geom-tail-bound}
Let $\{X_i\}_{i\in[n]}$ be independent random geometric variables with parameter $p_i\geq 0$, and let $X := \sum_{i=1}^{n} X_i$, 
and $p := \min_{i\in[n]}\{p_i\}$. 
If $\sum_{i=1}^{n} p_i^{-2} \leq s \leq \infty$ then for any $\lambda \geq 0$ it holds that 
\[
\Pr\left(X \geq \mathrm{E}[X] + \lambda \right)
    \leq e^{-\frac{1}{4}\min\left(\frac{\lambda^2}{s},\lambda p\right)}.
\]
\end{lemma}

We then have the following corollary for a single variable.
\begin{corollary}\label{cor:geom-tail-bound}
Let $X$ be a geometric random variable with parameter $p>0$, then for any real number $c>0$ and any natural number $n\geq e^{1/c}$, it holds $\Pr\left(X \geq \frac{c\ln{n}+1}{p}\right)\leq n^{-c/4}$.
\end{corollary}
\begin{proof}
Recall $\mathrm{E}[X]=1/p$, thus
applying Lemma~\ref{lem:sum-geom-tail-bound} with $\lambda:=\frac{c\ln{n}}{p}$ and $s:=\frac{1}{p^2}$ gives
\begin{align*}
\Pr\left(X \geq \frac{c\ln{n}+1}{p}\right)
    &\leq e^{-\frac{\min\left((c\ln{n})^2,c\ln{n}\right)}{4}}\\
    &= e^{-\frac{c\ln{n}}{4}} 
    = n^{-\frac{c}{4}}. \qedhere
\end{align*}
\end{proof}

In our analysis below, we also need the following inequalities which are summarised in the next two lemmas.

\begin{lemma}\label{lem:floor-and-ceil}
    For any real number $x>0$, it holds that
    $2\lfloor x \rfloor - \lfloor(3/2)x\rfloor + 1 \geq \lceil x/2\rceil - 1$.
\end{lemma}
\begin{proof}
    Applying the well-known inequality $\lfloor x+y\rfloor\leq \lfloor x\rfloor + 
    \lfloor y\rfloor+1$ which holds for any real numbers $x$ and $y$, to 
    $\lfloor(3/2)x\rfloor = \lfloor x + x - x/2\rfloor$ twice gives 
    $\lfloor(3/2)x\rfloor\leq \lfloor x \rfloor + \lfloor x \rfloor + \lfloor - 
    x/2\rfloor + 2 = 2\lfloor x \rfloor - \lceil x/2\rceil + 2$. Putting this
    back to the original statement concludes the proof. 
\end{proof}

\begin{lemma}\label{lem:change-exponent}
    For any natural number $i\geq 5$, it holds that 
    $(3/2)^i/2 - 1\geq (4/3)^i/2$.
\end{lemma}
\begin{proof}
    It suffices to show that $f(i):=(3/2)^i - (4/3)^i\geq 2$ for
    $i\geq 5$. It is easy to see that $f(i)$ is monotonically increasing in $i$ 
    by looking at its derivative, therefore for $i\geq 5$, we have 
    $f(i)\geq f(5)=(3/2)^5-(4/3)^5\geq 3.379>2$.
\end{proof}

\section{Runtime Analysis of QD-GA on the APSP} 
We first consider the optimisation progress achievable through mutation only. 
Particularly, we refer to \qdga (Algorithm~\ref{alg:QDGA}) with $p_c=0$ as \qdea. 
The following theorem shows that \qdga with a constant probability of carrying out mutation, thus including \qdea as a special case, solves APSP efficiently and bounds its optimisation time depending on the structural parameters $\Delta$ and $\ell$ of the given input.
\begin{theorem}
    \label{thm:QDGA-only-mutation}
The optimisation time of \qdga with $p_c=1- \Omega(1)$ is $O(n^2\Delta\max\{\ell,\log{n}\})$ in expectation and with probability $1-o(1)$. In particular, the statement on the optimisation time also holds for \qdea.
\end{theorem}

\begin{proof}
We first estimate the number of iterations $T_{st}$ to optimise an arbitrary cell 
$M_{st}$. 
Let $I=(s=v_1,\ldots,v_k=t)$ be any shortest path 
of this cell, then let $X_i$ be
the number of iterations in which cell $M_{s v_i}$ is optimised but 
$M_{s v_{i+1}}$ is not yet optimised, for any $i\in [k-1]$, so $T_{st}= \sum_{i=1}^{k-1}X_i$.
Since multiple elementary operations are allowed in one mutation, it is possible to optimise longer paths before the shorter ones, in other words some $X_i$ can take value zero.
However, for 
an upper bound on the optimisation time it suffices to consider the case where the path is extended by adding only one correct edge at a time. 
That is, to optimise $M_{s v_{i+1}}$, 
it suffices to pick the solution in the optimised cell $M_{s v_i}$ as parent, \ie probability
$1/(n(n-1))$, note that this solution has an equal weight to that of the path $(v_1,\ldots,v_i)$. 
Then the mutation is applied with only one elementary operation, 
\ie with probability $(1-p_c)/e$ for the distribution $\mathrm{Pois}(1)$, where 
the edge $(v_i,v_{i+1})$ is chosen for adding to the parent, \ie with probability
$1/|E_{st}|\geq 1/(2\Delta)$. 
The obtained offspring therefore has equal weight to that of the shortest path $(v_1,\ldots,v_{i+1})$
thus it is used to update and hence optimise $M_{s v_{i+1}}$, and so 
variable $X_i$ is stochastically
dominated by a geometric random variable $Y_i$ with parameter $p:=(1-p_c)/(2e n^2\Delta)$
regardless of $i$ and of the target path $I$. Furthermore, 

$$|I|=k-1\leq \ell \leq \max\{\ell,5\log{n}\}=:\ell'$$
and therefore
$$T_{st}= \sum_{i=1}^{k-1}X_i\preceq \sum_{i=1}^{\ell'} Y_i =: Y.$$
By linearity of
expectation $\mathrm{E}[Y]=\frac{2e n^2\Delta}{1-p_c} \cdot \ell'$, thus applying 
Lemma~\ref{lem:sum-geom-tail-bound} 
with $\lambda := \frac{2(6-e) n^2\Delta \ell'}{1-p_c}$ gives
\begin{align*}
\Pr\left(T_{st}\geq \frac{12 n^2\Delta \ell'}{1-p_c}\right)
  &\leq \Pr(Y\geq \mathrm{E}[Y] + \lambda)\\
  &\leq 
  e^{-\frac{1}{4}
      \min\left(\frac{\lambda^2}{(1/p)^2\ell'},
      \lambda p\right)}.
\end{align*}
Note that $\lambda p=(6/e-1)\ell'$ while $\lambda^2/((1/p)^2 \ell')
=(6/e-1)^2 \ell'>\lambda p$, thus 
the probability that $M_{st}$ is not optimised after $\frac{12 n^2\Delta \ell'}{1-p_c}
=:\tau$
iterations is at most
$
e^{-(6/e-1)\ell'/4}
$.

We now consider all the $n(n-1)$ cells. Given $\ell'\geq 5\log{n}=5\ln{n}/\ln{2}$, 
by a union bound the probability that not all cells are optimised after 
$\tau$ iterations is at most 
\begin{align*}
n(n-1)e^{-(6/e-1)\ell'/4} \leq n^2 e^{-5(6/e-1)\ln{n}/(4\ln{2})} < n^{-1/6}
\end{align*}
since $(30/e-5)/(4 \ln 2) > 1/6$. 
Thus with probability $1-o(1)$, all cells are optimised in $\tau = O(n^2\Delta\ell')$ 
iterations since $1-p_c=\Omega(1)$. Particularly this holds regardless of the initial 
population, thus the expected optimisation time is at most $(1+o(1))O(n^2\Delta\ell')$.
\end{proof}

Note that the running time of \qdea heavily depends on the characteristics of the graph. Particularly, for $\Delta=\Theta(n)$ and $\ell=\Theta(n)$ we obtain an upper bound of $O(n^4)$ on the runtime of the algorithm. 
A matching lower bound of $\Omega(n^4)$ can be obtained for $\Delta=\Theta(n)$ and $\ell=\Theta(n)$
by considering the complete directed graph $K_n$ as follows. It is defined as $K_n=(V,E)$ with $V:=\{v_1, \ldots , v_n\}$, $E:=\{(u,v) \mid u,v \in V, u \neq v\}$, and weights $w(v_i,v_j) = 1$ if $j-i=1$ and $w(v_i,v_j) = n$ otherwise (see~\citet{DBLP:journals/tcs/DoerrHK12}). Using Theorem 11 in~\citet{DBLP:journals/tcs/DoerrHK12}, we can obtain the following lower bound.
\begin{theorem}
\label{thm:lower-bound-mutation}
   The optimisation time of \qdea on $K_n$ with the above weight function is $\Omega(n^4)$ with probability $1-o(1)$ and therefore, the expected optimisation time is $\Theta(n^4)$. 
\end{theorem}

We now consider 
the case where 
$p_c$ is a constant in $(0,1)$.
This implies that both crossover and mutation are enabled. We show the following upper bound which points out that crossover speeds up the optimisation of our QD algorithm.

\begin{theorem}\label{thm:QDGA-mixing-mutation+crosover}
The optimisation time of \qdga with constant $p_c \in (0,1)$ is $O(n^{3.5} \sqrt{\log n})$ in expectation and with probability $1-o(1)$. 
\end{theorem}
\begin{proof}
    When a cell is optimised, it can admit multiple shortest paths of various
    cardinalities but with the same optimised weight. After that Algorithm~\ref{alg:Update} can still 
    update the cell with these equivalently optimal solutions. To argue about 
    the optimisation time, we therefore consider for each cell one of its shortest paths with the largest 
    cardinality as the \emph{representative} optimal solution (representative for short) of the cell.
    
    Let $m:= \lfloor{\sqrt{n \log(n)}}\rfloor$. 
    We divide the run into $\lfloor \log_{3/2} (\ell/m) \rfloor + 1 \leq \lfloor \log_{3/2} (n/m)\rfloor + 1$ 
    phases and a phase $i$ for $i \in [0, \lfloor \log_{3/2} (n/m)\rfloor]$ ends when all cells with 
    representatives of cardinality at most $\lfloor m \cdot (3/2)^i \rfloor$ have been optimised. 
    By Theorem~\ref{thm:QDGA-only-mutation} and given $1-p_c$ is a constant,  we see that every cell 
    $M_{st}$ with representative of cardinality at most $m$ contains 
    one of its shortest path 
    in time $O(n^{3.5}\sqrt{\log(n)})$ with probability $1-o(1)$ and in expectation. Therefore, phase~0 
    with the goal of optimising cells with representatives of cardinality at most $m$ 
    relies on mutation steps, 
    lasts for at most $O(n^{3.5}\sqrt{\log(n)})=:\tau_0$ iterations and only has failure probability of $o(1)$.
    
    Consider a \emph{further phase} $1+i$ for $i \geq 0$, and we rely on crossover steps. During such a phase, 
    all cells with representatives of cardinality at most $\lfloor(m \cdot 3/2)^i\rfloor$ have been optimised 
    while the optimisation of those with cardinality 
      $k \in [\lfloor m \cdot (3/2)^i\rfloor + 1, \lfloor m \cdot (3/2)^{i+1}\rfloor]$ is underway, and we refer 
    to the latter cells as \emph{target} cells. 
    For a target cell, assume that $I=(v_1,\ldots,v_{k+1})$ is its representative solution, then for 
    any integer $j \in [k+1-\lfloor m \cdot (3/2)^i \rfloor,\lfloor m \cdot (3/2)^i\rfloor +1]$ the paths 
    $I_{v_1 v_j}=(v_1,\ldots,v_j)$ and $I_{v_j v_{k+1}}=(v_{j},\ldots,v_{k+1})$ 
    are optimal solutions for cells $M_{v_1 v_j}$ and $M_{v_j v_{k+1}}$ respectively. 
    Furthermore, those paths have the same cardinalities as the corresponding representatives 
    of those cells, because
    otherwise (if there exists a shortest path with larger cardinality then) $I$ is not 
    the representative of $M_{v_1 v_{k+1}}$ and this contradicts our assumption.
    These imply that $M_{v_1 v_j}$ and $M_{v_j v_{k+1}}$ have already been 
    optimised since the cardinalities of $I_{v_1 v_j}$ and $I_{v_j v_{k+1}}$ 
    are at most $\lfloor m \cdot (3/2)^i \rfloor$. Thus picking solutions in those cells as parents and then applying crossover, \ie with probability 
    $p_c/(n^2(n-1)^2) \geq p_c/(n^4),$
    optimise the cell $M_{v_1 v_{k+1}}$ by creating either solution $I$ or 
    an equivalently optimal solution. The number of such pairs of cells (or parents) 
    is the number of possible integers $j$ which is at least 
  \begin{align*}
    &\lfloor m \cdot (3/2)^i \rfloor +1 -(k+1-\lfloor m \cdot (3/2)^i\rfloor)  + 1\\
    &= 2\lfloor{m \cdot (3/2)^i}\rfloor - k + 1\\
    &\geq 2 \lfloor{m \cdot (3/2)^i}\rfloor  - \lfloor m \cdot (3/2)^{i+1}  \rfloor + 1=:\xi_i.
  \end{align*}
Therefore, the number of iterations to optimise an arbitrary target cell
in phase $i+1$ is stochastically dominated by a geometric random variable 
with parameter 
$
p_c \xi_i/ n^4 =: p_i$.

Applying Corollary~\ref{cor:geom-tail-bound}
with $c=9$ for this variable implies that with probability at most $n^{-9/4}$, 
an arbitrary target cell is not optimised after $\tau_i := (9\ln{n}+1) n^4/(p_c \xi_i)$ 
iterations. By a union bound on at most $n(n-1)$ target cells, the probability 
that the phase is not finished after $\tau_i$ iterations is at most 
$n(n-1) n^{-9/4} \leq n^{-1/4}$. Then by a union bound on at most 
  $\lfloor \log_{3/2} (n/m)\rfloor=:q+1$
further phases the probability that all further phases are not finished after 
$
\tau := \sum_{i=0}^{q} \tau_i
$
iterations is at most $n^{-1/4}\log_{3/2} (n/m) 
=o(1)$. Thus combing with phase~$0$, with probability $1-2\cdot o(1)$, all cells 
are optimised in time $\tau_0+\tau$,
and if this does not happen, we can repeat the argument, thus the expected
running time of the algorithm to complete all phases is at most $(1+o(1))(\tau_0+\tau)$.

Note that $\tau_0$ is already $O(n^{3.5}\sqrt{\log{n}})$ and that $\tau = 
O(n^4 \log{n})\sum_{i=0}^{q} \xi_i^{-1}$ since 
$p_c$ is constant, thus it suffices to show that 
$\sum_{i=0}^{q} \xi_i^{-1} = O((\sqrt{n \log(n)})^{-1})$ 
to complete the proof.
We have for $n$ sufficiently large and due to $x-1 < \lfloor{x}\rfloor \leq x$
that 
\begin{align*}
\sum_{i=0}^{q} \xi_i^{-1} 
&= \sum_{i=0}^{q} \frac{1}{2\lfloor{m \cdot (3/2)^i}\rfloor - \lfloor{m \cdot (3/2)^{i+1}}\rfloor + 1}\\
&\leq \sum_{i=0}^{q} \frac{1}{2 (m \cdot (3/2)^i - 1) - m \cdot (3/2)^{i+1} + 1}\\
&= \sum_{i=0}^{q} \frac{1}{2 m \cdot (3/2)^i - 3m/2 \cdot (3/2)^{i}-1}\\
&= \sum_{i=0}^{q} \frac{1}{m/2 \cdot (3/2)^i-1}
 = \sum_{i=0}^{q} \frac{2}{m \cdot (3/2)^i-2}\\
&\leq \sum_{i=0}^{q} \frac{2}{m \cdot (3/2)^i-2 \cdot (3/2)^i} \\
&= \left(\frac{2}{m-2}\right) \sum_{i=0}^{q} \frac{1}{(3/2)^i} \leq \left(\frac{2}{m-2}\right) \sum_{i=0}^{\infty} \left(\frac{2}{3}\right)^i\\
&= O\left(m^{-1}\right) = O\left((\sqrt{n \log(n)})^{-1}\right).
       & \qedhere
\end{align*}
\end{proof}

The previous upper bound can be improved using the refined analysis carried out in~\citep{DoerrT09} which sets the cut-off length for mutation only to  $m=(n \log n)^{1/4}$ and shows that longer paths are constructed by a combination of crossover and mutation. This leads to the following improved upper bound for \qdga.

\begin{theorem}\label{thm:QDGA-mixing-mutation+crosover-improved}
The optimisation time of \qdga with constant $p_c \in (0,1)$ is $O(n^{3.25} \sqrt[4]{\log{n}})$ in expectation and with probability $1-o(1)$. 
\end{theorem}

\begin{proof}[Sketch of proof]
Using previous arguments for mutation only, in time $O(n^{3.25} \sqrt[4]{\log{n}})$ all shortest paths up to length $m=(n \log n)^{1/4}$ are constructed with probability $1-o(1)$. Using Lemma 3(c) in \citep{DoerrT09}, the time to obtain from a map containing all shortest path up to length $(3/2)^i \cdot m$ a map containing all shortest paths up to length $(3/2)^{i+1} \cdot m$ by a combination of crossover and mutation is  $O(n^{3.25} \sqrt[4]{\log{n}})\cdot (3/2)^{-i/3}$ with probability $1-O(n^{-c})$ for a positive constant $c$. Summing up over at most $\log_{3/2}(n)$ values of $i$, we get
$$O(n^{3.25} \sqrt[4]{\log{n}}) \sum_{i=1}^{\log_{3/2}(n)} (3/2)^{-i/3} = O(n^{3.25} \sqrt[4]{\log{n}})
$$
(bounding the geometric series by the constant $1/(1- ((2/3)^{1/3}))\leq 8$) with probability $1-o(1)$. The statement on the expectation follows as each phase of length $O(n^{3.25} \sqrt[4]{\log{n}})$ is successful with probability $1-o(1)$ and an expected number of at most $2$ phases is required.
\end{proof}
It should be noted that the bound given in Theorem~\ref{thm:QDGA-mixing-mutation+crosover-improved} is tight and a matching lower bound can be obtained by considering the graph $G_n$ considered in Theorem~4 in \citep{DoerrT09}.

\section{Speed-ups through Improved Operators}
We now consider improvements through feasible parent selection for crossover 
%
as inspired by \citet{DBLP:journals/tcs/DoerrJKNT13}.
In this selection, the first parent is selected from a cell of $M$ chosen 
uniformly at random. Assume this selection is successful, and this first chosen cell 
is $M_{st}$ then the second parent is selected uniformly at random from a cell in 
the row $t$ but excluding column $s$ (and of course excluding column $t$). 
There is a chance that this selection fails, \ie an empty cell 
in one of those steps, 
in this case no crossover occurs and
the algorithm skips directly to the next iteration.

\begin{theorem}\label{thm:QDGA-only-new-crossover}
The optimisation time of QD-GA 
with $p_c=\Omega(1)$
and using the above improved crossover operator 
is $O (n^3\log{n})$ 
in expectation and with probability $1-o(1)$. 
\end{theorem}

\begin{proof}
We use the same notion of representative of a cell as in the proof of Theorem~\ref{thm:QDGA-mixing-mutation+crosover}. 
The run of the algorithm is divided into $\lfloor \log_{3/2} \ell \rfloor \leq \lfloor \log_{3/2} n\rfloor$ 
phases and a phase $i$ ends when all cells with representative solutions of cardinality 
at most $\lfloor (3/2)^i \rfloor$ have been optimised. Note that phase $1$ is completed 
at initialisation, thus we only look at phases $i+1$ for $i\geq 1$. During such a phase, 
all cells with representatives of cardinality at most $\lfloor(3/2)^i\rfloor$ have been optimised 
while the optimisation of those with cardinality 
$k \in [\lfloor (3/2)^i\rfloor + 1, \lfloor (3/2)^{i+1}\rfloor]$ is underway,
and we refer to the latter cells as \emph{target} cells. 

For a target cell, assume that $I=(v_1,\ldots,v_{k+1})$ is its representative 
solution, 
then for any integer $j\in [k+1-\lfloor(3/2)^i \rfloor,\lfloor(3/2)^i\rfloor +1]$ the paths 
    $I_{v_1 v_j}=(v_1,\ldots,v_j)$ and $I_{v_j v_{k+1}}=(v_{j},\ldots,v_{k+1})$ 
    are optimal solutions for cells $M_{v_1 v_j}$ and $M_{v_j v_{k+1}}$ respectively. 
Furthermore, those paths have the same cardinalities as the corresponding representatives 
of those cells, because
otherwise 
$I$ is not 
the representative of $M_{v_1 v_{k+1}}$ and this contradicts our assumption.
These imply that $M_{v_1 v_j}$ and $M_{v_j v_{k+1}}$ have already been 
optimised since the cardinalities of $I_{v_1 v_j}$ and $I_{v_j v_{k+1}}$ 
are at most $\lfloor (3/2)^i \rfloor$. Thus picking solutions in those cells as 
parents then applying crossover, \ie with probability $p_c/(n(n-1)(n-2))$, 
optimise the cell $M_{v_1 v_{k+1}}$ by creating either solution $I$ or 
an equivalently optimal solution. The number of such pairs of cells (or parents) 
is the number of possible integers $j$ which is 
at least 
\begin{align*}
&\lfloor (3/2)^i \rfloor +1 -(k+1-\lfloor (3/2)^i\rfloor)  + 1\\
    &= 2\lfloor(3/2)^i \rfloor - k + 1\\
    &\geq 2 \lfloor (3/2)^i \rfloor  - \lfloor (3/2)^{i+1}  \rfloor  + 1
     =:\xi_i.
\end{align*}
Therefore, the number of iterations to optimise an arbitrary target cell
in phase $i+1$ is stochastically dominated by a geometric random variable 
with parameter 
$
p_c \xi_i/ n^3 =: p_i$. 

Applying Corollary~\ref{cor:geom-tail-bound}
with $c=9$ for this variable implies that with probability at most $n^{-9/4}$, 
an arbitrary target cell is not optimised after $\tau_i := (9\ln{n}+1) n^3/(p_c \xi_i)$ 
iterations. By a union bound on at most $n(n-1)$ target cells, the probability 
that the phase is not finished after $\tau_i$ iterations is at most 
$n(n-1) n^{-9/4}=n^{-1/4}$. Then by a union bound on at most $\lfloor \log_{3/2} n\rfloor$
phases the probability that all phases are not finished after 
$
\tau := \sum_{i=1}^{\lfloor\log_{3/2} n\rfloor} \tau_i
$
iterations is at most $n^{-1/4}\log_{3/2} n 
=o(1)$. Thus with probability $1-o(1)$, all cells are optimised in time $\tau$,
and if this does not happen, we can repeat the argument, thus the expected
running time of the algorithm is at most $(1+o(1))\tau$. Note that 
\[
\tau = O(n^3 \log{n})\sum_{i=1}^{\lfloor \log_{3/2} n\rfloor} \xi_i^{-1}
\]
since $p_c=\Omega(1)$
thus it suffices 
to show that $\sum_{i=1}^{\lfloor \log_{3/2} n\rfloor} \xi_i^{-1} = O(1)$ to complete the proof. 
It is easy to check that $\xi_i \geq 1$ for $i\in[1,4]$, thus we separate the sum 
by the first 
four summands and for the rest we estimate $\xi_i$ by 
Lemmas~\ref{lem:floor-and-ceil} and \ref{lem:change-exponent} with $x=(3/2)^i$:
\begin{align*}
\sum_{i=1}^{\lfloor \log_{3/2} n\rfloor} \xi_i^{-1}
   & \leq 4 + \sum_{i=5}^{\lfloor \log_{3/2} n\rfloor} \xi_i^{-1}\\
   & \leq 4 + \sum_{i=5}^{\lfloor \log_{3/2} n\rfloor}\frac{1}{\frac{1}{2}\cdot \left(\frac{3}{2}\right)^i-1}  &\text{by Lemma~\ref{lem:floor-and-ceil}}\\ 
   & \leq 4 + \sum_{i=0}^{\infty}\frac{2}{\left(\frac{4}{3}\right)^i}  &\text{by Lemma~\ref{lem:change-exponent}}\\
   & = 4 + 2/(1-3/4) = O(1).
       & \qedhere
\end{align*}
\end{proof}

Combining the Theorems~\ref{thm:QDGA-only-mutation} and \ref{thm:QDGA-only-new-crossover} gives the following result for the so-called fast QD-APSP algorithm.

\begin{theorem}\label{thm:QDGA-fast}
QD-GA
  with 
  constant $p_c \in(0,1)$
  and with the improved crossover operator, this setting is referred as the fast QD-APSP algorithm, 
%
%
has optimisation time $O(\min\{\Delta n^2 \cdot \max\{\ell, \log n\}, n^3 \log n\})$ on the APSP in expectation and with probability $1-o(1)$.
\end{theorem}

Fast QD-APSP takes the advantage of both operators, for example on 
balanced $k$-ary trees where $k=O(1)$ and $\ell=O(\log{n})$, the expected running time of the algorithm is only $O(n^2\log{n})$. 
However,
there also exist graphs 
where a tight optimisation time $\Omega(n^3\log{n})$ is required with high probability and also in expectation, 
even on a simple weight function that only assigns unit weights.

\begin{theorem}\label{thm:QDGA-fast-worst-case}
    There is an instance where the Fast QD-APSP algorithm 
    requires optimisation time $\Omega(n^3 \log n)$ with probability $1-o(1)$ and in expectation.
\end{theorem}

\begin{proof}
Let $n\geq 2$ be any natural number, we consider the graph $G=(V,E)$ with $|V|=5n+1$
which consists of:
\begin{itemize}
    \item $n$ paths $P_i=(u_i,c_i,v_i)$ of cardinality $2$ where $i\in[n]$,
    \item a complete bipartite graph $K_{n,n}$ with two disjoint sets of nodes $A,B$ where 
    $|A|=|B|=n$, and for each pair $(a,b)\in A \times B$, we have $(a,b)\in E$ and $(b,a)\in E$,
    \item a node $c$, and the remaining edges of $E$ are: $(c,u_i)$ and 
    $(v_i,c)$ for each $i \in[n]$; $(c_i,a)$, $(b,c_i)$ for all $i\in[n]$, 
    all $a\in A$, and all $b\in B$,
\end{itemize} 
and 
the unit weight function $w\colon E\rightarrow \{1\}$. 
Figure~\ref{fig:hard-for-fast-qdga} shows an example of this graph for $n=2$. 
It is easy to see 
    that $G$ is strongly connected, 
    and that 
    its diameter 
    is $\ell=5=O(1)$ (\eg the cardinality of the shortest path from node $u_i$ to $v_j$ where $i\neq j$).
The maximum degree is 
        $\Delta
            = 3n
            = \Theta(|V|)$
    which corresponds to the degree of a node in $K_{2n}$,
    but 
    node $c_i$ of any path $P_i$ also has degree $2n+2=\Theta(|V|)$.
\begin{figure}[h]\centering

\begin{tikzpicture}[x=3em,y=3em,scale=0.9]
    \tikzstyle{vertex}=[draw,thick,fill=white,circle,minimum size=1.8em,align=center];
    \tikzstyle{edge}=[draw,black,thick,-stealth];    

    \fill[gray,rounded corners=5] (0.4,2.45) rectangle (2.6,-0.45);
    \node[white] at (2,2) {$P_1$};    
    \node[vertex] at (1,2) (u1) {$u_1$};
    \node[vertex] at (2,1) (c1) {$c_1$};
    \node[vertex] at (1,0) (v1) {$v_1$};
    \path[edge,white] (u1) -- (c1);
    \path[edge,white] (c1) -- (v1);

    \fill[gray,rounded corners=5] (0.4,5.45) rectangle (2.6,2.55);
    \node[white] at (2,5) {$P_2$};    
    \node[vertex] at (1,5) (u2) {$u_2$};
    \node[vertex] at (2,4) (c2) {$c_2$};   
    \node[vertex] at (1,3) (v2) {$v_2$};
    \path[edge,white] (u2) -- (c2);
    \path[edge,white] (c2) -- (v2);
    
    \node[vertex] at (-1,2.5) (c) {$c$};
    \path[edge] (c) -- (u1);
    \path[edge] (c) -- (u2);
    \path[edge] (v1) -- (c);
    \path[edge] (v2) -- (c);

    \fill[gray,rounded corners=5] (3.4,4.1) rectangle (6.6,0.9);
    \node[black] at (5,4.6) {$K_{2,2}$};
    \fill[gray!50,rounded corners=5] (3.6,3.9) rectangle (6.4,3.1);
    \node[black] at (6.8,3.5) {$A$};
    \fill[gray!50,rounded corners=5] (3.6,1.9) rectangle (6.4,1.1);
    \node[black] at (6.8,1.5) {$B$};
    \node[vertex] at (4,3.5) (a1) {$a_1$};
    \node[vertex] at (4,1.5) (b1) {$b_1$};    
    \node[vertex] at (6,3.5) (a2) {$a_2$};
    \node[vertex] at (6,1.5) (b2) {$b_2$};    
    \path[edge] (a1.-80) -- (b1.80);
    \path[edge] (b1.100) -- (a1.-100);
    \path[edge] (a2.-80) -- (b2.80);
    \path[edge] (b2.100) -- (a2.-100);    
    \path[edge] (a1.-35) -- (b2.125);
    \path[edge] (b2.145) -- (a1.-55);
    \path[edge] (a2.-145) -- (b1.55);
    \path[edge] (b1.35) -- (a2.-125);

    \path[edge] (c1) -- (a1);
    \path[edge] (c1) -- (a2);
    \path[edge] (c2) edge[bend left] (a1);
    \path[edge] (c2) edge[bend left] (a2);
    \path[edge] (b1) edge[bend left] (c1);
    \path[edge] (b2) edge[bend left] (c1);
    \path[edge] (b1) -- (c2);
    \path[edge] (b2) -- (c2);

\end{tikzpicture}
\caption{Graph $G$ for $n=2$, 
all edges have weight $1$.}
\label{fig:hard-for-fast-qdga}
\end{figure}
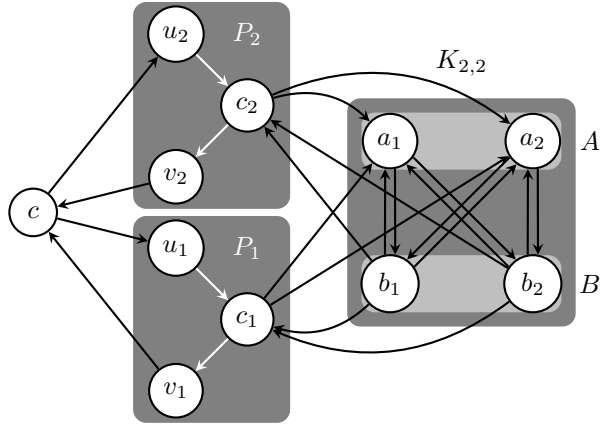

The unique shortest path between $u_i$ and $v_i$ is $P_i$, therefore we argue 
that with high probability Fast QD-APSP requires 
    $\Omega(|V|^3 \log|V|)=\Omega(n^3\log{n})$ 
iterations to optimise the $n$ cells $M_{u_i v_i}$ where $i\in[n]$. After
initialization, none of these cells are optimised since $P_i$ has cardinality $2$. 
Then in every iteration, in order to optimise a cell $M_{u_i v_i}$ by crossover
cells $M_{u_i c_i}$ and $M_{c_i v_i}$ must be selected as parents, and this
occurs with probability $p_c/((5n+1)5n(5n-1))$. If mutation is used instead,
\ie with probability $1-p_c$, either $M_{u_i c_i}$ and $M_{c_i v_i}$ cells 
need to be selected as the parent, \ie with probability $2/((5n+1)5n)$, then
one proper edge needs to be selected among $2n+2$ of those that are connected 
to $c_i$ in an elementary operation to optimise $M_{u_i v_i}$, \ie probability 
at most $1/(2n+2)$. 
Note that it is possible to optimise $M_{u_i v_i}$ by first 
optimising some cell with a shortest path of larger cardinality and later mutating 
it back to $M_{u_i v_i}$. However, 
both shortening a shortest path that contains $P_i$ as a sub-path (\eg converting path $(u_i,c_i,v_i,c)$ to $P_i$)
and mutating a shortest path that contains an edge of $P_i$ (\eg converting path $(u_i,c_i,a)$ to $P_i$) 
require at least 
making the above correct elementary operation in the first place thus we can ignore these events.
So overall, the probability of optimising $M_{u_i v_i}$ by either crossover
or mutation is at most 
$$2/((5n+1)5n(2n+2))\leq 1/(25n^3).$$
This upper bound 
holds independently of $i$ as far as the cell $M_{u_i v_i}$ has not yet been 
optimised. 

Consider now the first $\tau:=(1/5)(25n^3-1)\ln{n}$ iterations, the 
probability that an arbitrary cell $M_{u_i v_i}$ is not optimised during these 
iterations is at least, here using $(1-1/x)^{x-1}\geq 1/e$ for $x=25n^3$,
\[
\left(1-\frac{1}{25n^3}\right)^{(1/5)(25n^3-1)\ln{n}}
  \geq n^{-1/5}.
\]
Therefore, the probability that at least a cell among the $n$ cells $M_{u_i v_i}$ 
is not optimised after $\tau=\Omega(n^3\log{n})$ iterations is at least
\[
1-(1-n^{-1/5})^{n} 
    \geq 1 - e^{-n^{-1/5}\cdot n}
    = 1-o(1).
\]
This implies that Fast QD-APSP requires $\Omega(n^3\log{n})$ fitness evaluations
to optimise $G$ with probability $1-o(1)$. This also implies that the expectation is in $\Omega(n^3 \log n)$.
\end{proof}

\section{Conclusions}
Computing diverse sets of high quality solutions is important in various areas of artificial intelligence.
Quality diversity algorithms have received a lot of attention in recent years due to their ability of tackling problems from a wide range of domains. We contributed to the theoretical understanding of these algorithms by providing the first analysis of a classical combinatorial optimisation problem that seeks multiple solutions in a natural behaviour space. Our analysis for the APSP points out the working behaviour of Map-Elites for this problems and shows that these algorithms can provably solve it efficiently. Afterwards, we presented speed up techniques that provide significantly better upper bounds than the standard Map-Elites approach. 

Establishing a rigorous theoretical foundation and obtaining a better understanding of QD has the potential to improve practical applications across various domains.
We hope that this work serves as a stepping stone towards the development of more efficient quality diversity algorithms and provides a basis for understanding QD algorithms for more complex planning problems.

\section*{Acknowledgments}

This work has been supported by the Australian Research Council (ARC) through grants DP190103894 and FT200100536.

\end{document}